\documentclass{scrartcl}

\usepackage{hyperref}
\usepackage{url}
\usepackage{natbib}           
\usepackage{amsmath}
\usepackage{amsthm}
\usepackage{amssymb}
\usepackage{sectsty}
\usepackage{graphicx}
\graphicspath{{./figures/}}
\usepackage{subfigure}
\usepackage{color}
\usepackage{url}
\usepackage{todonotes}
\usepackage{hyperref}
\usepackage{commath}
\usepackage{amsfonts}
\usepackage{mdframed}
\usepackage{enumerate}
\usepackage{bold-extra}
\usepackage{nicefrac}
\usepackage{wrapfig}
\usepackage{caption}
\usepackage{enumitem}

\usepackage{algorithm}
\usepackage{algorithmic}
\usepackage{empheq}
\definecolor{myblue}{rgb}{.8, .8, 1}
\newcommand*\mybluebox[1]{%
\colorbox{myblue}{\hspace{1em}#1\hspace{1em}}}

\DeclareMathOperator*{\argmin}{arg\,min}

\def\RR{{\mathbb R}}

\def\prox{\text{Prox}}
\def\eprox{{\emph{\text{Prox}}}}

\newmdtheoremenv{alg}{Algorithm}
\newmdtheoremenv{theo}{Theorem}

\newtheorem{lemma}{Lemma}

\newtheorem{definition}{Definition}

\author{
Fabian Pedregosa,\quad
\href{mailto:f@bianp.net}{f@bianp.net} \\
\'Ecole Normale Sup\'erieure / INRIA Sierra project-team \\
Paris, France
}
\title{On the convergence rate of the three operator splitting scheme}
\date{}

\begin{document}

\maketitle
\begin{abstract}
  The three operator splitting scheme was recently proposed by~\citep{davis2015three} as a method to optimize composite objective functions with one convex smooth term and two convex (possibly non-smooth) terms for which we have access to their proximity operator. In this short note we provide an alternative proof for the sublinear rate of convergence.
\end{abstract}

\section{Introduction}

We consider the problem of optimizing a composite objective function of the form
\begin{empheq}[box=\mybluebox]{equation}\tag{OPT}\label{eq:obj_fun}
  \argmin_{x \in \RR^p} \underbrace{f(x)}_{\text{smooth}} + \underbrace{g(x) + h(x)}_{\text{nonsmooth}} \quad,
\end{empheq}
where $f$ is differentiable with Lipschitz continuous gradient (also known as $L$-smooth) and $g, h$ are both convex but not necessarily differentiable. Furthermore, we assume that we have access to the proximal operator of $g$ and $h$ but not necessarily to that of $g+h$.

Optimization problems of this form are pervasive in machine learning. It includes for example the class of functions that can be efficiently optimized using proximal gradient methods~\citep{beck2009gradient}, such as $\ell_1$-penalized logistic regression, group lasso, etc. Since it allows for two (instead of one) proximable terms, it also includes other penalized models that are costly to optimize using proximal gradient descent, such as group lasso with overlap~\citep{jacob2009group} or multi-dimensional total variation~\citep{barbero2014modular}.

The three-operator splitting scheme~\citep{davis2015three} was recently proposed to solve problems of this form. It requires access to the gradient of $f$ and the proximal operator of $g$ and $h$. The recurrence iterates for this algorithm is given by
\begin{empheq}[box=\mybluebox]{equation}\tag{ALG}\label{eq:recurrence}
  \begin{aligned}
  x &= \prox_{\gamma g}(y),  \\
  z &= \prox_{\gamma h}(2 x - y - \gamma \nabla f(x)) \\
  y^+ &= y - x + z  \\
\end{aligned}
\end{empheq}
This algorithm can be seen as a generalization of some well known algorithms. For example, when $g$ vanishes, then $x = y$ in the first line and the algorithm defaults to the common proximal gradient descent~\citep{beck2009gradient} ($x^+ = \prox_{\gamma g}(x - \gamma \nabla f(x))$). On the other hand, when $f=0$, then the update becomes identicall to Douglas-Rachford algorithm~\citep{lions1979splitting,eckstein1992douglas}. For more than 2 proximal operators this algorithm defaults to the generalized forward-backward method of~\citet{raguet2013generalized}.

In this note we provide an alternative analysis of the sublinear convergence rate derived in~\citep{davis2015three}. In particular, we consider convergence with respect to the gradient mapping instead of function values. This way it is possible to derive non-ergodic rates. This proof is loosely based on that of~\citet{he2015convergence} for the Douglas-Rachford algorithm.

\section{Analysis}\label{scs:analysis}

In this section we provide convergence rates for the three operator splitting. For this, we make the following assumptions on the optimization problem:
\begin{itemize}
  \item[\bfseries (A1)] $f$ is $L$-smooth (i.e. differentiable with $L$-Lipschitz gradient) and convex.
  \item[\bfseries (A2)] The functions $g$ and $h$ are convex.
  \item[\bfseries (A3)] Problem~\eqref{eq:obj_fun} admits at least one solution.
\end{itemize}

Like for the Douglas-Rachford algorithm, the convergence analysis of this algorithm is challenging due to its non-monotonicity. That is, the iterates do not always decrease the function value. In contrast with proximal gradient descent, at a given iteration the objective function need not even be finite: in the case that $g$ and $h$ are the indicator function of two sets, the solution is guaranteed to converge to an element in the intersection (assuming this is non-empty) , but for intermediate steps the iterates might as well (and often do) lie outside this intersection, making the objective value infinity.

To overcome this difficulty~\citet{davis2015three} proved convergence rates in terms of a modified objective function in which $g$ and $h$ are evaluated at different points. Here instead we take an alternative approach. We first define the ``gradient mapping'' of this operator, that is, a generalization of the of the gradient for gradient descent methods and prove convergence rates based on this criterion.

\begin{definition}[Gradient mapping]\label{def:gradient_mapping}
  We define the gradient mapping of the three operator splitting \eqref{eq:recurrence} as
\begin{equation}
  \begin{aligned}
    G(y) &:= \frac{1}{\gamma}(y - y^+) = \frac{1}{\gamma}(x - z) \\
    &= \frac{1}{\gamma}(\eprox_{\gamma g}(y) - \eprox_{\gamma h}( 2 \eprox_{\gamma g}(y) - y - \gamma \nabla f(\eprox_{\gamma g}(y)))
  \end{aligned}
\end{equation}
\end{definition}
We highlight two convenient aspects of this notation. First is that using $G$, the updates of \eqref{eq:recurrence} can in the simple form $y^+ = y - \gamma G(y)$, highlighting its similarity with a gradient descent rule. The second aspect is that this is indeed a generalization of the gradient since when $h=g=0$ then it coincides with the gradient: $G(y) = \nabla f(y)$. If only one nonsmooth term is zero (say $h$), then it defaults to the prox-grad map~\citep{beck2009gradient}: $G(y) = \nicefrac{1}{\gamma}(x - \prox_{\gamma h}(x - \gamma \nabla f(x)))$. The prox-grad map is often used to obtain rates of convergence in nonconvex composite optimization~\citep{ghadimi2016mini}.

We will also extensively use the notion of nonexpansive operator:
\begin{definition}[Nonexpansive operator]
We recall that an operator $T: \RR^p \to \RR^p$ is non-expansive whenever the following is verified for all $y, \tilde{y}$ in the domain:
  $$
  \langle T(y) - T(\tilde{y}), y - \tilde{y}\rangle \geq \|T(y) - T(\tilde{y})\|^2
  $$
\end{definition}
The importance of this notion comes from the fact that the proximal operators of a convex function are nonexpansive (see e.g.~\citep{bauschke2011convex}).
\\

The next lemma stablished a relationship between minimizers of \eqref{eq:obj_fun} and points in which the gradient mapping vanishes:
\begin{lemma}[Fixed point characterization]\label{lemma:fixed_point} Let $x^*$ be a minimizer of \eqref{eq:obj_fun}. Then there exists $y^*$ such that $x^* = \eprox_{\gamma g}(y^*)$ and $G(y^*)= 0$.
\end{lemma}
\begin{proof}
  Let $x^*$ be a minimizer of~\eqref{eq:obj_fun}. Then by the first order optimality conditions there exists $u \in \partial g(x^*)$ and $v \in \partial h(x^*)$ such that $u + v = - \nabla f(x^*)$. We define $y^* = x^* + \gamma u $ and we have the inclusion
  $$
  \frac{1}{\gamma}(y^* - x^*) = u \in \partial g(x^*) \iff x^* = \prox_{\gamma g}(y^*) \quad,
  $$
  by definition of proximal operator. This proves $x^* = \prox_{\gamma g}(y^*)$. We still need to prove that $G$ vanishes at $y^*$. By definition of $y^*$ and the property $u + v = - \nabla f(x^*)$, we further have the inclusion
  $$
  \frac{1}{\gamma}(x^* - y^* - \gamma \nabla f(x^*)) = v \in \partial h(x) \iff x^* = \prox_{\gamma h}(2 x^* - y^* - \gamma \nabla f(x^*)),
  $$
  which replacing in Definition~\eqref{def:gradient_mapping} and using $x^* = \prox_{\gamma g}(y^*)$ yields $G(y^*) = 0$.
  \end{proof}

{\bfseries Outline of the proof}. The proof is structured as follows. We start by a technical lemma (Lemma~\ref{lemma:master_inequality}), in which we prove a property of the operator $\gamma G$ that we will use in the following. After this, we prove two key lemmas: Lemma~\ref{lemma:decreasing_gradient} and \ref{lemma:distance_decreasing}. Finally, sublinear convergence is proved in Theorem~\ref{thm:sublinear_convergence}.

\begin{lemma}[Master inequality]\label{lemma:master_inequality} For any pair of arbitrary vectors $y, \tilde{y} \in \RR^p$, the operator $G$ verifies the following inequality:
$$
\langle \gamma G(y) - \gamma G(\tilde{y}), y - \tilde{y} \rangle \geq \|\gamma G(y) - \gamma G(\tilde{y})\|^2 + \gamma \langle \nabla f(x) - \nabla f(\tilde{x}), z - \tilde{z}\rangle
$$
where $x, z$ (resp. $\tilde{x}, \tilde{z})$ are computed from $y$ (resp. $\tilde{y}$) as in \eqref{eq:recurrence}.
\end{lemma}
\begin{proof}
  By the non-expansiveness of $g$ and $h$ we have the following inequalities:
  We will also make use of the following inequalities, derived from the strict nonexpansiveness of the proximal operator:
\begin{subequations}
\begin{align}
   \langle x - \tilde{x}, y - \tilde{y}\rangle &\geq \|x - \tilde{x}\|^2    ~,\label{eq:strict_nonexpansiveness}\\
   \langle 2 x - y - \gamma \nabla f(x) - 2 \tilde{x} + \tilde{y} + \gamma \nabla f(\tilde{x}) , z - \tilde{z}\rangle &\geq \|z - \tilde{z}\|^2  \label{eq:strict_nonexpansiveness_2}
\end{align}
\end{subequations}
Using these we can write the following sequence of inequalities:
  $$
  \begin{aligned}
    \langle \gamma G(y) - \gamma G(\tilde{y}), y - \tilde{y} \rangle &= \langle x -z - \tilde{x} + \tilde{z}, y - \tilde{y} \rangle \\
    &= \langle x -\tilde{x}, y - \tilde{y} \rangle  - \langle z -
    \tilde{z} , y - \tilde{y} \rangle \\
    &\geq \|x -\tilde{x}\|^2 - \langle z -
    \tilde{z}, y - \tilde{y} \rangle \qquad \text{ (by Eq.~\eqref{eq:strict_nonexpansiveness})}\\
    &\geq  \|x - \tilde{x}\|^2 + \|z -
      \tilde{z}\|^2 - 2\langle x- \tilde{x}, z -
      \tilde{z}\rangle \\
      &\qquad + \gamma \langle \nabla f(x) - \nabla f(\tilde{x}), z - \tilde{z} \rangle \qquad \text{ (by Eq.~\eqref{eq:strict_nonexpansiveness_2})} \\
      &= \gamma^2 \|G(y) - G(\tilde{y})\|^2 + \gamma \langle \nabla f(x) - \nabla f(\tilde{x}), z - \tilde{z}\rangle \quad,
  \end{aligned}
  $$
  where the last equality follows completing the square and the definition of $G(y)$. The result is obtained dividing both sides by $\gamma$.
  \end{proof}

%

{\bfseries Remark}. Note the similarity between the inequality in the previous lemma and the definition of nonexpansiveness: indeed, when $\nabla f= 0$ the previous lemma states that the operator $G$ is nonexpansive. Unfortunately, this is not true for general $f$.
\\

Now we will prove that the norm of the gradient mapping forms a monotonically decreasing sequence:
\begin{lemma}[Monotonicity of gradient mapping]\label{lemma:decreasing_gradient}
  For any $y \in \RR^p$ it is verified that
  $$
  \|G(y^+)\|^2 \leq \|G(y)\|^2 - \left(1 - \frac{\gamma L}{2}\right)\|G(y^+) - G(y)\|^2 \quad.
  $$
\end{lemma}
\begin{proof}
  Consider the identity
  \begin{equation*}\label{eq:sublinear_1}
    \begin{aligned}
    \|G(y^+)\|^2 &= \|G(y^+) - G(y) + G(y)\|^2 \\
    &= \|G(y)\|^2 + \|G(y^+) - G(y)\|^2 + 2 \langle G(y^+) - G(y), G(y) \rangle \\
    \end{aligned}
  \end{equation*}
  Using Lemma~\ref{lemma:master_inequality} with $\tilde{y} = y^+$ to majorize the last term gives the following inequality
  $$
  \|G(y^+)\|^2 \leq\|G(y)\|^2 - \|G(y^+) - G(y)\|^2 - \frac{2}{\gamma} \langle \nabla f(x) - \nabla f(x^+), z - z^+ \rangle \quad.
  $$

  This can be further bounded by using the properties of $L$-smooth functions. For the second term we have the following sequence of inequalities:
  \begin{equation*}
    \begin{aligned}
      & - \langle \nabla f(x) - \nabla f(x^+), z - z^+ \rangle = - \langle \nabla f(x) - \nabla f(x^+), x - \gamma G(y) - x^+ + \gamma G(y^+) \rangle \\
      &\qquad\leq - \frac{1}{L}\|\nabla f(x) - \nabla f(x^+)\|^2 + \langle \nabla f(x) - \nabla f(x^+), \gamma G(y) - \gamma G(y^+) \rangle\\
      &\qquad\qquad \text{ (properties of $L$-smooth functions, \citep[Theorem 2.1.5]{nesterov2004introductory})} \\
      &\qquad\leq - \frac{1}{L}\|\nabla f(x) - \nabla f(x^+)\|^2 + \frac{\beta}{2}\|\nabla f(x) - \nabla f(x^+) \| + \frac{1}{2 \beta} \|\gamma G(y) - \gamma G(y^+) \|^2 \\
      &\qquad\qquad \text{ (Cauchy-Schwarz and Young's inequality, $\beta$ is arbitrary)}\\
      &\qquad = \frac{\gamma^2 L}{4}\| G(y) - G(y^+) \|^2
    \end{aligned}
  \end{equation*}
  where in the last equality we have chosen $\beta=2 / L$. Replacing into the previous inequality yields the desired result.
\end{proof}

{\bfseries Remark}. The previous lemma proves that the gradient mapping for this algorithm is decreasing. Note however that the objective values of the Douglas-Rachford (and hence for the three operator splitting) method can be non-monotonous~\citep{bauschke2014rate}.

The second ingredient of our proof is to show that not only the gradient mapping is decreasing, but that also the distance to optimum decreases.
\begin{lemma}[Decreasing distance to optimum]\label{lemma:distance_decreasing} Let $y^*$ be a point in which the gradient mapping vanishes (which is guaranteed to exist by assumption (A3) and Lemma~\ref{lemma:fixed_point}). Then the following inequality is verified:
  $$
  \|y^+ - y^*\|^2 \leq \|y - y^*\|^2 - \gamma^2\left(1 - \frac{\gamma L}{2}\right)\|G(y)\|^2
  $$
\end{lemma}
\begin{proof}
  Because $G(y) = 0 \implies x^* = z^*$.
\begin{equation*}
  \begin{aligned}
    \|y^+ - y^*\|^2 &= \|y - \gamma G(y) - y^*\|^2 = \|y - y^*\|^2 + \gamma^2 \|G(y)\|^2 - 2 \gamma \langle G(y), y - y^*\rangle \\
    &\leq\|y - y^*\|^2 - \gamma^2 \|G(y)\|^2 - 2 \gamma \langle \nabla f(x) - \nabla f(x^*), z - z^* \rangle
  \end{aligned}
\end{equation*}
where in the last inequality we have used Lemma~\ref{lemma:master_inequality} with $\tilde{y} = y^*$.
We can further bound the last term using properties of $L$-smooth functions as:
  $$
  \begin{aligned}
    & -\langle \nabla f(x)- \nabla f(x^*), z -
   z^* \rangle\\
   &\qquad= -\langle  \nabla f(x)- \nabla f(x^*), x -
  x^* \rangle + \langle \nabla f(x) - \nabla f(x^*), z -
  x \rangle \\
  &\qquad\qquad \text{ (using $G(y^*) = 0 \iff x^* = z^*$)} \\
  &\qquad\leq - \frac{1}{L} \|\nabla f(x) - \nabla f(x^*)\|^2 - \langle \nabla f(x) -  \nabla f(x^*), z -
  x \rangle \\
  &\qquad\qquad \text{ (properties of $L$-smooth functions, \citep[Theorem 2.1.5]{nesterov2004introductory})} \\
  &\qquad\leq - \frac{1}{L} \|\nabla f(x) - \nabla f(x^*)\|^2 + \frac{\beta}{2}\| \nabla f(x) - \nabla f(x^*)\|^2 + \frac{1}{2\beta}\|z -
  x \|^2 \\
  &\qquad\qquad \text{ (Cauchy-Schwarz and Young's inequality, $\beta > 0$ is arbitrary)}\\
  &\qquad=  \frac{\gamma^2 L}{4}\|G(y)\|^2 \\
  &\qquad\qquad \text{ (choosing $\beta=2 / L$ and by definition of $G$)}
  \end{aligned}
  $$
  Plugging it into the previous inequality yields the result of this lemma.
\end{proof}

\begin{theo}[Sublinear convergence]\label{thm:sublinear_convergence}
  Let $\{y^k\}$ be the sequence generated by \eqref{eq:recurrence}. For any integer $k > 0$ and $\gamma < 2 / L$, we have
 $$
\|G(y^k)\|^2 \leq \frac{2 \|y^0 - y^*\|^2}{ \gamma^2 (2 - \gamma L)(k+1)  } = \mathcal{O}(1/k) \quad.
 $$
\end{theo}
\begin{proof}
By summing the inequality of Lemma~\ref{lemma:distance_decreasing} from $y=y^0$ to $y=y^k$ we have
$$
\|y^{k+1} - y^*\|^2 \leq \|y^0 - y^*\|^2 - \gamma^2 \left(1 - \frac{\gamma L}{2} \right)\sum_{i=0}^k \|G(y^i)\|^2
$$
from where
$$
\gamma^2\left( 1 - \frac{\gamma L}{2}\right) \sum_{i=0}^k \|G(y^i)\|^2 \leq \|y^0 - y^*\|^2
$$
Furthermore, since by Lemma~\ref{lemma:decreasing_gradient} the gradient mapping is decreasing, we can further lower bound bound $\|G(i)\|$ by $\|G(k)\|$ and obtain
$$
(k+1) \gamma^2\left( 1 - \frac{\gamma L}{2}\right)  \|G(y^k)\|^2 \leq \|y^0 - y^*\|^2 \quad.
$$
Finally, if $\gamma < 2/L$ then all the terms in the right hand side are positive and we can rearrange terms to obtain the desired inequality.
\end{proof}

This result also sheds some light on the appropriate step size $\gamma$. Minimizing $\gamma^2(2 - \gamma L)$ with respect to $\gamma$ gives $\gamma= \frac{4}{3L}$, which is close to the theoretical optimal for gradient descent, $1/L$.

\clearpage
\section{Acknowledgements}

The author acknowledges financial support from the ``Chaire Economie des Nouvelles Donnees'', under the auspices of Institut Louis Bachelier, Havas-Media and Universit\'e Paris-Dauphine (ANR 11-LABX-0019).

\bibliography{biblio}{}
%
%
%
%

\end{document}